\documentclass[twoside]{article}
\usepackage{aistats2016}

\usepackage{hyperref}
\usepackage{url}
\usepackage{natbib} 
\setcitestyle{numbers,open={[},close={]}}

\usepackage{amsmath} 
\usepackage{amssymb} 
\usepackage{amsthm}
\usepackage{subcaption,graphicx,caption,subfloat,booktabs,multirow}

\usepackage{pgfplots}\pgfplotsset{compat=1.4} 

\graphicspath{{./}{figures/}}

\usepackage{algorithmic}

 \newtheorem{lemma}{Lemma}[section]

 \newtheorem{algorithm}[lemma]{Algorithm}

\DeclareMathOperator*{\argmin}{argmin}

\begin{document}

%

%

\twocolumn[

\aistatstitle{Automatic Inference of the Quantile Parameter}

\aistatsauthor{ Karthikeyan Natesan Ramamurthy \And Aleksandr Y. Aravkin \And Jayaraman J. Thiagarajan }

\aistatsaddress{IBM Thomas J. Watson Research Center \And University of Washington \And Lawrence Livermore National Laboratory } ]

\begin{abstract}
Supervised learning is an active research area, with numerous applications in diverse fields such as data analytics, computer vision, speech and audio processing, and image understanding. In most cases, the loss functions used in machine learning assume symmetric noise models, and seek to estimate the unknown function parameters. However, loss functions such as quantile and quantile Huber generalize the symmetric $\ell_1$  and Huber losses to the asymmetric setting, for a fixed quantile parameter. In this paper, we propose to jointly infer the quantile parameter and the unknown function parameters, for the asymmetric quantile Huber and quantile losses. We explore various properties of the quantile Huber loss and implement a {\it convexity certificate} that can be used to check convexity in the quantile parameter. When the loss if convex with respect to the parameter of the function, we prove that it is biconvex in both the function and the quantile parameters, and propose an algorithm to jointly estimate these. Results with synthetic and real data demonstrate that the proposed approach can automatically recover the quantile parameter corresponding to the noise and also provide an improved recovery of function parameters. To illustrate the potential of the framework, we extend the gradient boosting machines with quantile losses to automatically estimate the quantile parameter at each iteration. 
\end{abstract}

\section{INTRODUCTION} 
\label{sec:intro}

In supervised learning, we are interested in estimating the functional relationship between the random \textit{input} variables $\mathbf{x} = \{x_1, \ldots, x_p\}$ and the \textit{response} or \textit{output} variable $y$. The training sample $\{\mathbf{x}_i,y_i\}_{i=1}^n$ is provided and denoting the function that maps $\mathbf{x}$ to $y$ as $f(\mathbf{x})$, we estimate it such that the expected value of the \textit{loss}, $L(y,f(\mathbf{x}))$, is minimized over the joint distribution of all $(y, \mathbf{x})$.

Symmetric losses are used in a range of cases, including regression and robust regression. 
While the quadratic loss is often the method of choice, 
it is not robust with respect to the presence of outliers in the data
\citep{Hub,Gao2008,Aravkin2011tac,Farahmand2011} and may have
difficulties in reconstructing fast system dynamics, e.g. jumps in
the state values \citep{Ohlsson2011}. 
In these cases, the Huber penalty
\citep{Hub}, Vapnik's $\epsilon$-insensitive loss, used in support
vector regression \citep{Vapnik98,Hastie01}. 

These methods are sufficient when the data is homogenous;
however, when the data is heterogeneous, merely estimating the conditional mean
is insufficient, as estimates of the standard errors are often biased.
To comprehensively analyze such heterogeneous datasets, \textit{quantile
regression}~\cite{KB78} has become a popular alternative. 
In quantile regression, one studies the effect of explanatory variables on the
entire conditional distribution of the response variable, rather than just on its mean value. 
Quantile regression is therefore well-suited
to handle heterogeneous datasets~\cite{Buchinsky:1994}.
A sample of recent works in areas such as computational biology \cite{Zou08},
survival analysis~\cite{KG01}, and economics~\cite{KH01} serve as testament to
the increasing popularity of quantile regression.
Furthermore, quantile regression is \emph{robust to
outliers}~\cite{Koenker:2005}: the quantile regression loss is a piecewise
linear ``check function'' that generalizes the absolute deviation loss for
median estimation to other quantiles, and shares the robustness properties of
median estimation in the presence of noise and outliers. 
A huberized extension of the quantile loss has recently been proposed
in the context of high dimensional regression~\cite{aravkin2014qh}. 

The quantile loss and its huberized version are parameterized by a scalar
variable. While state-of-the-art relies on cross-validation methods to obtain 
reasonable values for this parameter, it is intuitively clear that we can infer 
the degree of asymmetry in the loss in addition to the standard parameters of 
interest. Especially in the large data setting, a large residual vector may inform
the degree of asymmetry. 

The paper proceeds as follows. In Section~\ref{sec:QHQ} we introduce the quantile
and quantile Huber loss, and specify their asymmetric parametrization. 
We then design an inference scheme for the asymmetry parameter by exploiting 
the statistical distributions corresponding to the asymmetric losses of interest. In particular, we show 
that the normalization constant, which can be ignored for standard quantile and quantile Huber regression, 
plays a key role both for theoretical and practical aspects of the inference scheme. 
We present a joint optimization method to obtain regression estimates and the 
asymmetry parameter in Section~\ref{sec:Alg}. Finally, we present numerical examples 
that exploit the proposed framework in the context of gradient boosted machines 
in Section~\ref{sec:GBM}.

\section{QUANTILE AND QUANTILE Huber LOSS}
\label{sec:QHQ}
Quantile loss is an important learning tool for the high dimensional inference, 
and a `huberized' version has recently been shown to perform even better in some settings~\cite{aravkin2014qh}.
In particular, while the {\it quantile Huber} shares the asymmetric features of the quantile loss, 
it is smooth at the origin. 

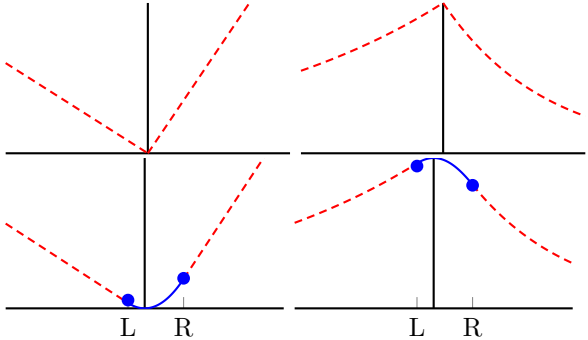
\begin{figure}[t]
\centering
\begin{minipage}{0.49\textwidth} 
\begin{tikzpicture}
  \begin{axis}[
    thick,
    width=.45\linewidth, height=2cm,
    xmin=-2,xmax=2,ymin=0,ymax=1,
    no markers,
    samples=50,
    axis lines*=left, 
    axis lines*=middle, 
    scale only axis,
    xtick={0},
    ytick={0},
    ] 
\addplot[red,domain=-2:0,densely dashed]{-.3*x};
\addplot[red,domain=0:+2,densely dashed]{.7*x};
  \end{axis}
\end{tikzpicture}
\begin{tikzpicture}
  \begin{axis}[
    thick,
    width=.45\linewidth, height=2cm,
    xmin=-2,xmax=2,ymin=0,ymax=1,
    no markers,
    samples=50,
    axis lines*=left, 
    axis lines*=middle, 
    scale only axis,
    xtick={0},
    ytick={0},
    ] 
\addplot[red,domain=-2:0,densely dashed]{exp(-(-.3*x))};
\addplot[red,domain=0:+2,densely dashed]{exp(-(.7*x))};
  \end{axis}
\end{tikzpicture}


\begin{tikzpicture}
\vspace{-.05 in}
  \begin{axis}[
    thick,
    width=.44\linewidth, height=2cm,
    xmin=-2,xmax=2,ymin=0,ymax=1,
    no markers,
    samples=100,
    axis lines*=left, 
    axis lines*=middle, 
    scale only axis,
    xtick={-.24,.56},
    xticklabels={L, R},
    ytick={0},
    ] 
\addplot[red,domain=-2:-2*0.3*0.4,densely dashed]{0.3*abs(x) - 0.4*0.3^2};
\addplot[blue,domain=-2*0.3*0.4:2*(1-0.3)*0.4]{0.25*x^2/0.4};
\addplot[red,domain=2*(1-0.3)*0.4:2,densely dashed]{(1-0.3)*abs(x) - 0.4*(1-0.3)^2};
\addplot[blue,mark=*,only marks] coordinates {(-.24,0.0550) (0.56,0.20)};
  \end{axis}
\end{tikzpicture}
\begin{tikzpicture}
\vspace{-.05 in}
  \begin{axis}[
    thick,
    width=.44\linewidth, height=2cm,
    xmin=-2,xmax=2,ymin=0,ymax=1,
    no markers,
    samples=100,
    axis lines*=left, 
    axis lines*=middle, 
    scale only axis,
    xtick={-.24,.56},
    xticklabels={L, R},
    ytick={0},
    ] 
\addplot[red,domain=-2:-2*0.3*0.4,densely dashed]{exp(-(0.3*abs(x) - 0.4*0.3^2))};
\addplot[blue,domain=-2*0.3*0.4:2*(1-0.3)*0.4]{exp(-(0.25*x^2/0.4))};
\addplot[red,domain=2*(1-0.3)*0.4:2,densely dashed]{exp(-((1-0.3)*abs(x) - 0.4*(1-0.3)^2))};
\addplot[blue,mark=*,only marks] coordinates {(-.24,0.9464851) (0.56, 0.8187308)};
  \end{axis}
\end{tikzpicture}
\end{minipage}

\caption{Quantile $(\tau = 0.3)$  loss (top left) and associated density (top right);  
quantile Huber $(\tau = 0.3)$ loss (bottom left) and associated density (bottom right). 
The quantile Huber loss is obtained by smoothing the quantile loss at the origin.}
\label{fig:PLQFig}
\end{figure}
In this section, we define these two loss functions, and describe their characteristics and properties. 
To enable the inference on the {\it quantile parameter} $\tau$, we develop a statistical
interpretation for quantile and quantile Huber densities, and find normalization constants for these densities.
We also illustrate the key role the normalization constants play in quantile inference. 

For the scalar residual $r = y-f(\mathbf{x})$, the quantile loss (or {\it check function}) is defined as follows:
\begin{equation}
\label{en:quant_loss}
q_{\tau}(r) = (-\tau+1 [{r \geq 0}]) r,
\end{equation} where $1[.]$ is the indicator function. 
Note that the quantile loss is same as the $\ell_1$ loss when $\tau = 0.5$. 
We focus on two features: robustness to outliers (heavy tails), and 
 sparsity of residual (non-smoothness at the origin). 
A sparse residual corresponds to exact fitting of some of the data, 
and this is unlikely to be desirable in most regression settings. 
If we apply Moreau-Yosida smoothing (c.f.~\cite{RTRW}), we obtain use the \textit{quantile Huber}, which matches 
the asymmetric slopes of the quantile outside the interval $[-\kappa\tau,
\kappa(1-\tau)]$, and is quadratic on this interval:
\begin{equation}
\label{quantileHuber}
\small
\rho_\tau(r) = \begin{cases}
\tau |r| - \frac{\kappa \tau^2}{2} & \text{if} \text{ } r < -\tau\kappa,\\
\frac{1}{2\kappa}r^2 & \text{if} \text{ } r\in [-\kappa \tau, (1-\tau)\kappa],\\ 
(1-\tau) |r| - \frac{\kappa(1-\tau)^2}{2}, & \text{if} \text{ } r > \quad (1-\tau)\kappa.
\end{cases}
\end{equation} 
In a data fitting context, the quantile Huber 
is more permissive, allowing small errors for all residuals, in contrast to the quantile loss
which forces a portion of the data to be fit exactly. 
The quantile Huber becomes the Huber loss when $\tau = 0.5$. 
Figure \ref{fig:PLQFig} shows the quantile $(\tau = 0.3)$ and quantile Huber $(\tau = 0.3)$ loss functions.  
As $\kappa \rightarrow 0$, the quantile Huber converges to the quantile loss, so we can view the quantile loss as a 
special case of the quantile Huber. 
The loss for a multi-dimensional residual $\mathbf{r}$ is defined as the sum of losses for individual dimensions for both quantile and quantile Huber losses, i.e., $q_{\tau}(\mathbf{r}) = \sum_{i=1}^N q_{\tau}(r_i)$ and $\rho_\tau (\mathbf{r}) = \sum_{i=1}^N \rho_\tau (r_i)$


\subsection{Quantile Huber Density and Normalized Loss}
In order to perform inference of the quantile parameter $\tau$, we build a correspondence between densities that correspond to 
the penalties of interest. 
For a scalar residual, the we define the density function for the quantile Huber loss as follows: 
\begin{equation}
p (r | \tau) = \frac{1}{c(\tau)} e^{-\rho_{\tau} (r)},
\end{equation} where $c(\tau)$ is the normalization constant: 
\begin{equation}
c(\tau) = \int e^{-\rho_{\tau} (r)} dr.
\end{equation} 
Quantile Huber loss now corresponds to the negative log of the density, ignoring the normalization constant $c(\tau)$. 
If $\tau$ is fixed and known, ignoring $c(\tau)$ does not impact regression estimates. 
However, in our case we are interested in $\tau$, and it becomes essential to consider $c(\tau)$. 
The joint loss function $-\log(p(y | \tau))$ including the normalization constant is given by
\begin{equation}
\label{eq:jointRho}
\bar{\rho}_{\tau} (r) = \rho_{\tau} (r) + \log (c(\tau)).
\end{equation}The normalization constant $c$ can be obtained in 
closed form by simple integration: 
\begin{align}
\nonumber
c(\tau) =& \frac{1}{\tau} e^{-\frac{\kappa \tau^2}{2}} + \frac{1}{1-\tau} e^{-\frac{\kappa (1-\tau)^2}{2}}\\
&+ \sqrt{2 \pi \kappa} \left( F((1-\tau) \sqrt{\kappa}) - F(-\tau \sqrt{\kappa}) \right),
\end{align} where $F$ is the CDF of the standard normal distribution.

\subsubsection{Gradient and Hessian}
\label{sec:GradHess}

Optimization over $\tau$ can be done efficiently once we know the 
gradient and the Hessian of~\eqref{eq:jointRho}. We show the necessary 
derivations in this section, and present an algorithm for joint inference
of $(\tau, x)$ in Section~\ref{sec:Alg}.
The derivative of $\bar{\rho}_{\tau} (r)$ with respect to $\tau$ is
\begin{equation}
\frac{\partial \bar{\rho}_{\tau} (r)}{\partial \tau} =\frac{\partial \rho_{\tau} (r)}{\partial \tau}  + \frac{1}{c(\tau)} \frac{\partial c (\tau)}{\partial \tau}.
\end{equation} The first term can be immediately computed from~\eqref{quantileHuber}:
\begin{equation}
\frac{\partial \rho_{\tau} (r)}{\partial \tau}  = \begin{cases}
-r - \kappa \tau & \text{if} \text{ } r < -\tau\kappa,\\
0 & \text{if} \text{ } r\in [-\kappa \tau, (1-\tau)\kappa],\\ 
-r + \kappa(1-\tau), & \text{if} \text{ } r > \quad (1-\tau)\kappa.
\end{cases}
\end{equation} The derivative $c'(\tau)$ can also be computed in closed form: 
\begin{align}
\label{eqn:cp}
 c'(\tau) = -\frac{1}{\tau^2}e^{-\frac{\kappa \tau^2}{2}} + \frac{1}{(1-\tau)^2}e^{-\frac{\kappa (1-\tau)^2}{2}}.
\end{align} 

\begin{figure}
\centering
\includegraphics[width=7cm]{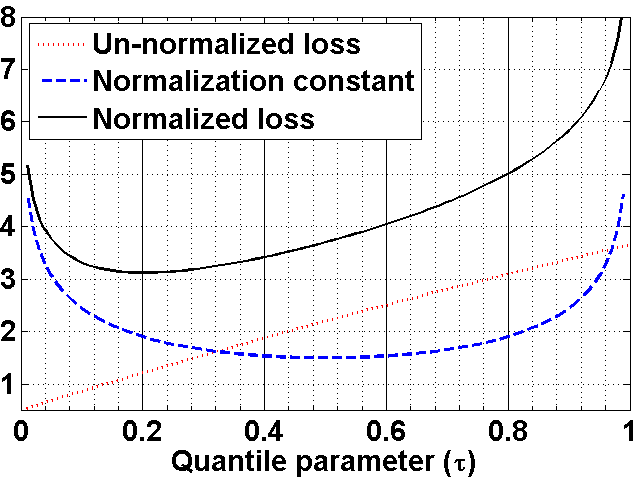}      
\caption{The concave un-normalized loss when added to the convex normalization constant, results in the final loss which is convex in $\tau$ for $\kappa = 1$. $10 \%$ of elements in the residual vector are positive, and the total loss is small for small $\tau$ values.}
\label{fig:lossfn}
\end{figure}

The second derivative of $\log(c(\tau))$ is given by 
\begin{align}
\label{eqn:Hess2}
(\log(c(\tau))''  = \frac{c(\tau) c''(\tau) -(c'(\tau))^2}{c^2(\tau)},
\end{align} where $c'(\tau)$ is given in (\ref{eqn:cp}) and
\begin{align}
\nonumber
c''(\tau) &=  \left(\frac{\kappa}{\tau}+\frac{2}{\tau^3} \right) e^{-\frac{\kappa \tau^2}{2}}\\
\nonumber 
&+\left(\frac{\kappa}{1-\tau}+\frac{2}{(1-\tau)^3} \right) e^{-\frac{\kappa (1-\tau)^2}{2}}.
\end{align} Further, since
\begin{equation}
\label{eqn:Hess1}
\frac{\partial^2 \rho_{\tau} (r)}{\partial \tau^2}  = \begin{cases}
- \kappa \tau & \text{if} \text{ } r \notin [-\kappa \tau, (1-\tau)\kappa],\\
0 & \text{otherwise},
\end{cases}
\end{equation} 
we know $\frac{\partial^2 \bar{\rho}_{\tau} (r)}{\partial \tau^2}$ is just the sum of (\ref{eqn:Hess1}) and (\ref{eqn:Hess2}). The positivity of this Hessian provides us a \textit{certificate of convexity} for $\bar{\rho}_{\tau} (r)$ with respect to $\tau$. This quantity depends on the Huber parameter $\kappa$. 
We empirically verified that minimum value of the second partial of $\bar \rho$ (over $\tau \in [0,1]$) decreases from $8$ to $6$ as $\kappa$ moves from $0$ to $1$, 
so the objective function is strongly convex in $\tau$ for all of these values.  
In contrast, the un-normalized loss, ${\rho}_{\tau} (r)$ given in (\ref{quantileHuber}) is clearly concave in $\tau$, with the minimum value of the second derivative given by $-\kappa$. Adding the normalization term gives rise to a convex problem in $\tau$, paving the way for an inference for both $x$ and $\tau$. 
This idea is demonstrated in Figure \ref{fig:lossfn}, where only $10 \%$ of values in a $1000$-dimensional residual vector were set to be positive. 
By definition, in this case the quantile Huber loss is expected to be small for small values of $\tau$. 
The un-normalized loss is concave, and the minimum is {\it always} $\tau = 0$ when negative elements dominate in the residual, and {\it always} $\tau = 1$ when positive elements dominate. This behavior in principle cannot recover a true $\tau$ strictly in the range $(0,1)$, and can also adversely affect inference in $x$.
When the normalization term is added, the loss becomes convex and the minimum is obtained in the lower range of $\tau$ values. In addition, as shown in Figure \ref{fig:lossfn_diffnoise}, the minimizer $\overline\tau$ for the normalized loss shifts smoothly across $(0,1)$ as the noise changes from mostly negative to mostly positive.

\begin{figure}
\centering
\includegraphics[width=7cm]{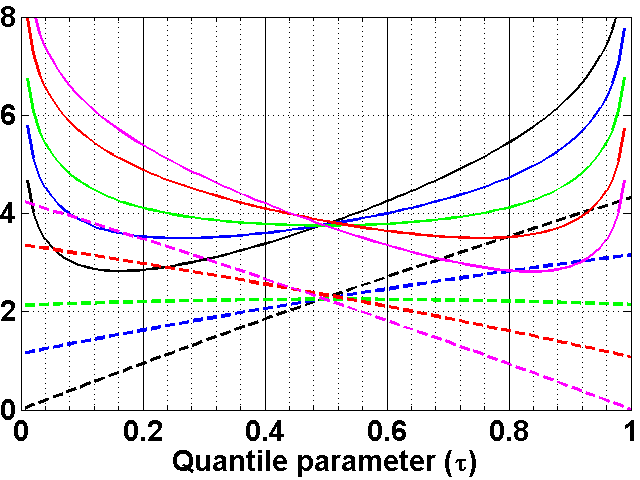}      
\caption{The un-normalized loss (dashed line) and normalized loss (solid line) for different percentages of positive values in the residual - $10 \%$ (black),  $25 \%$ (blue),  $50 \%$ (green),  $75 \%$ (red), and  $100 \%$ (magenta), for $\kappa = 1$. While minimizing the un-normalized loss can provide only a very coarse estimate of $\tau$ - either $0$ or $1$, using the normalized loss can provide a much finer estimate that changes smoothly with the actual $\tau$.}
\label{fig:lossfn_diffnoise}
\end{figure}

In the case of multidimensional residual $\mathbf{r}$, the total loss $\rho_\tau({\mathbf{r}})$ and the total normalized loss $ \bar{\rho}_{\tau} (\mathbf{r})$ is obtained as a sum over corresponding quantities at individual dimensions. Hence, the total gradient and Hessian are also obtained as a summation over the individual dimensions. 

{\bf Remark}: Quantile loss is a special case of quantile Huber, since quantile Huber converges (pointwise) to the quantile loss as $\kappa \rightarrow 0$. The loss function, normalization constant, gradient and Hessian of the quantile loss can be obtained in the limit $\kappa \rightarrow 0$ from the corresponding expressions for quantile Huber. As mentioned before, the second derivative of quantile Huber is positive for $\kappa = 0$ and $\tau \in [0,1]$. Hence it is clear that quantile loss is strongly convex and $\tau$ can be inferred in a similar manner, as described in the next Section.

\section{ALGORITHM}
\label{sec:Alg}
We are interested in joint inference on functional parameters 
$x$ and quantile paramaters $\tau$. 
Taking the normalization constant into account, the full 
inference problem is given by 
\begin{equation}
\label{eq:jointInf}
(\overline x, \overline \tau) = \argmin_{\mathbf{x},\tau} \left\{\rho_\tau(r(\mathbf{x})) + \log(c(\tau))\right\}.
\end{equation}
If the residual $r(\mathbf{x})$ is affine in $\mathbf{x}$, we can focus our analysis onto the tuple $(r,\tau)$.
Joint convexity of the objective~\eqref{eq:jointInf} in $(r,\tau)$ would be a sufficient condition 
for joint inference to be well posed. We show however that this is necessarily false. 
\begin{lemma}
\label{lem:notjointconv}
The function $f(r, \tau) = h(\tau) + \tau r$ cannot be jointly convex in $(r,\tau)$ for any $h$. 
\end{lemma} 
\begin{proof}
Without loss of generality, suppose we could find a twice differentiable $h(\tau)$ for the sum to be convex. 
Then $f$ is also twice differentiable, with hessian  
\[
\nabla^2 f = \begin{bmatrix}
\nabla^2 h(\tau) & 1 \\ 
1 & 0
\end{bmatrix}.
\]
This matrix is always indefinite, since its determinant is negative. This contradicts the 
necessary condition for convexity of a twice differentiable function. 
\end{proof}
Lemma~\ref{lem:notjointconv} can be immediately used to show that,  regardless of how large a quadratic lower 
bound one can find for $\log c(\tau)$,
the function~\eqref{eq:jointRho} cannot be jointly convex in $(r,\tau)$, and consequently in $(\tau, x)$.
However, $f(r, \tau)$ is {\it biconvex} for any convex $h(\tau)$. 
We therefore get a simple certificate of biconvexity for the joint inference problem. 
\begin{lemma}
\label{lem:biconvexity}
A sufficient condition for biconvexity of 
\begin{equation}
\label{eq:fullObj}
g(\mathbf{x}, \tau) := n\log(c(\tau)) + \sum_{i=1}^n \rho_\tau(y_i-\mathbf{a}_i^T \mathbf{x})  
\end{equation}
in $(\mathbf{x},\tau)$ is $\nabla^2 c(\tau) > \kappa$.
\end{lemma}
\begin{proof}
For fixed $\mathbf{x}$, hence fixed $\mathbf{r}$, $\rho_\tau(\mathbf{r})$ is concave in $\tau$, with Hessian $-n \kappa$. 
The necessary condition on $\log c(\tau)$ follows immediately. 
\end{proof}
The condition in Lemma~\ref{lem:biconvexity} can be easily verified numerically using 
the results of Section~\ref{sec:GradHess}.  Note also that Lemmas \ref{lem:notjointconv} and \ref{lem:biconvexity} will hold true as long as $\rho_\tau$ is convex in $\mathbf{x}$, although we proved them for the affine case. The simple structure of the objective in $\tau$ suggests an elegant 
scheme for joint optimization of~\eqref{eq:fullObj}. Since for each fixed $\mathbf{x}$, 
quantile inference boils down to minimizing a scalar convex objective, 
we can use {\it variable projection}. Specifically, define 
\begin{equation}
\label{eq:projObj}
\widetilde g(\mathbf{x}) = f(\mathbf{x}, \tau(\mathbf{x})),
\end{equation}
with $\tau(\mathbf{x}) = \argmin_\tau g(\mathbf{x},\tau)$. The problem fits into the scheme studied by~\cite{AravkinVanLeeuwen2012}, 
and in particular we have 
\begin{equation}
\label{eq:varProj}
\nabla \widetilde g(\mathbf{x}) = [\partial_{x_i} f(\mathbf{x}, \tau) | _{\tau = \tau(\mathbf{x})}]_{i=1}^p
\end{equation}
In other words, to compute the gradient of the projected function, one needs only compute the partials with respect to 
the elements of $\mathbf{x} \in \mathbb{R}^p$, and plug in the inferred value of $\tau$. 
The resulting algorithm has the flavor of an alternating minimization method, 
but in fact can be analyzed as a descent method for the projected function $\widetilde g(\mathbf{x})$. 

\begin{algorithm}
  \label{alg:joint}
  Minimize $g(\mathbf{x},\tau)$ in~\eqref{eq:fullObj} over $\tau \in [0,1]$.
  \begin{enumerate}
  \item Initialize $\mathbf{x}^0 = 0$, $k = 0$, $\mbox{err} = 1$, $\mathbf{H}_k = \mathbf{I}$.
  \item While $\mbox{err} > \epsilon$
  \begin{enumerate}
  \item Compute $\tau^k = \argmin_\tau g(\mathbf{x}^k, \tau)$
  \item Set $\mathbf{d}^k =  \mathbf{H}_k^{-1}[\partial_{x_i} g(\mathbf{x}, \tau^k)]_{i=1}^p$
  \item Set $\mathbf{x}^{k+1} = \mathbf{x}^k - \alpha_k \mathbf{d}^k$ using line search
  \item Update $\mathbf{H}^k$ so that $\mathbf{H}^k > 0$. 
  \item Set $\mbox{err} = \|\mathbf{d}^k\|$ and Increment $k$
  \end{enumerate} 
  \item Output $(\mathbf{x}^k, \tau^k)$.
           \end{enumerate}
\end{algorithm}
In particular, the matrix  $\mathbf{H}^k$ is the Hessian approximation for the projected function $\widetilde g(\mathbf{x})$,
and $\mathbf{d}^k$ is a descent direction for this function. 
There are many choices for Hessian update strategies. 
For large-scale problems,  the L-BFGS strategy can be very effective. 
Since $\tau$ lies in a compact set and $\rho_\tau(\mathbf{r}) \geq 0$, we have $\widetilde g(\mathbf{x})$ is bounded below, 
and Algorithm~\ref{alg:joint} is guaranteed to find a stationary point for~\eqref{eq:fullObj}.

\subsection{Demonstration}
We provide a demonstration of the proposed approach in a linear regression setting where the data is generated as $y_i = \mathbf{a}_i^T \mathbf{x} + \nu$. A total of $100$ samples were generated by fixing $\mathbf{a}_i$ as a $2-$dimensional vector realized from a Gaussian random distribution, and $\mathbf{x} = [2 \text{ } 5]$. The noise $\nu$ is generated from a Laplacian distribution. The sign of the elements of the noise are varied from $100 \%$ positive to $100 \%$ negative. The coefficient vector $\mathbf{a}$ is recovered using least squares. We also perform the joint recovery of coefficients and quantile parameter using the proposed approach. Average recovery performance is evaluated using MSE with respect the true coefficients for $100$ repetitions of the experiment and tabulated in Table \ref{tab:rec_perf}. Clearly, the proposed approach is superior to least squares in coefficient recovery at all noise conditions. Further, the recovered quantile parameter is also reflective of the noise added: $\tau$ changes from high to low as the noise changes from fully positive to fully negative.

\begin{table}[htbp]
  \centering
  \caption{Average performance (MSE) with least squares and the proposed joint inference approach for recovering the coefficient and the quantile parameter.}
    \begin{tabular}{cccc}
    \toprule
    \textbf{\% Positive} & \textbf{Least sq.} & \textbf{Proposed} & \textbf{$\tau$} \\
    \midrule
    100   & 0.255 & 0.036 & 0.788 \\
    80    & 0.270 & 0.143 & 0.719 \\
    60    & 0.231 & 0.135 & 0.587 \\
    40    & 0.260 & 0.157 & 0.422 \\
    20    & 0.248 & 0.103 & 0.289 \\
    0     & 0.263 & 0.039 & 0.212 \\
    \bottomrule
    \end{tabular}%
  \label{tab:rec_perf}%
\end{table}%

For the same data, we also perform coefficient recovery at individual quantiles using plain quantile Huber regression without the proposed inference scheme for $\tau$ and plot the mean normalized quantile penalty. This is shown in Figure \ref{fig:lossfn_quants}. Clearly the $\tau$ at which the loss is minimized match well with the $\tau$ inferred by our proposed algorithm in Table \ref{tab:rec_perf}.

\begin{figure}
\centering
\includegraphics[width=7cm]{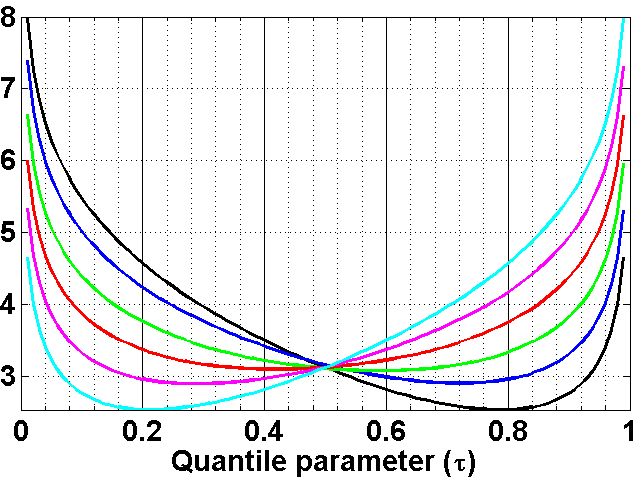}      
\caption{Mean normalized quantile penalty at various quantiles for different percentages of positive values in the residual - $0 \%$ (cyan),  $20 \%$ (magenta),  $40 \%$ (red),  $60 \%$ (green), $80 \%$ (blue), and $100 \%$ (black) for $\kappa = 1$. This is performed using plain quantile Huber regression without the proposed inference scheme. The minimum values are loss are respectively obtained at $\tau = \{0.79, 0.72, 0.59, 0.42, 0.29, 0.21\}$.}
\label{fig:lossfn_quants}
\end{figure}

\section{APPLICATION TO GRADIENT BOOSTED MACHINES}
\label{sec:GBM}

Gradient boosted machines (GBM) are a machine learning paradigm where the key idea is to assume that the unknown function $f$ is a linear combination of several \textit{base learners} \cite{friedman2001greedy}. The base learners will be greedily trained by setting their target response to be the negative gradient of the current loss with respect to the current prediction. The base learner can be imagined to be the ``basis function'' for the negative gradient. Concretely, if we assume the function of interest 
\begin{equation}
f(\mathbf{x}) = \sum_{j = 1}^m  \beta_j \psi_j(\mathbf{x} | \mathbf{z}_j)
\end{equation} where $\psi_j(\mathbf{x} | \mathbf{z})$ is the base learner parameterized by $\mathbf{z}$. The GBM proceeds by performing a stagewise greedy fit,
\begin{equation}
\label{eqn:greedy_update}
(\beta_j, \mathbf{z}_j) = \argmin_{\beta,\mathbf{z}} \sum_{i=1}^n 
L(y_i, f_{j-1} (\mathbf{x}_i)+ \beta \psi_j(\mathbf{x}_i | \mathbf{z})),
\end{equation} where $L$ is the loss function and $f_{j-1}$ is the estimate of the function obtained at the previous iteration,
\begin{equation}
f_{j-1}(\mathbf{x}) = \sum_{t = 1}^{j-1}  \beta_t \psi_t(\mathbf{x} | \mathbf{z}_t).
\end{equation} The estimate of base learner parameters at iteration $j$ is obtained by setting its response $\widetilde{\mathbf{y}}$ to be the negative gradient
\begin{equation}
\label{eqn:neg_grad}
\widetilde{y}_i = -\left[ \frac{\partial L(y_i,f(\mathbf{x}_i))} {\partial f(\mathbf{x}_i)} \right]_{f(\mathbf{x}) = f_{j-1}(\mathbf{x})}
\end{equation} for all $i = \{1, \ldots, n\}$. The parameter $\mathbf{z}_j$ is updated using
\begin{equation}
\label{eqn:zupdate}
(\gamma, \mathbf{z}_j) = \argmin_{\gamma,\mathbf{z}} \sum_{i=1}^n 
(\widetilde{y}_i -  \gamma \psi_i(\mathbf{x}_i | \mathbf{z}))^2.
\end{equation} The base learner coefficient $\beta_j$ is updated using (\ref{eqn:greedy_update}).

\begin{figure}[!h]
\centering
\begin{minipage}[b]{\linewidth}
 \centering
\includegraphics[width=7.5cm]{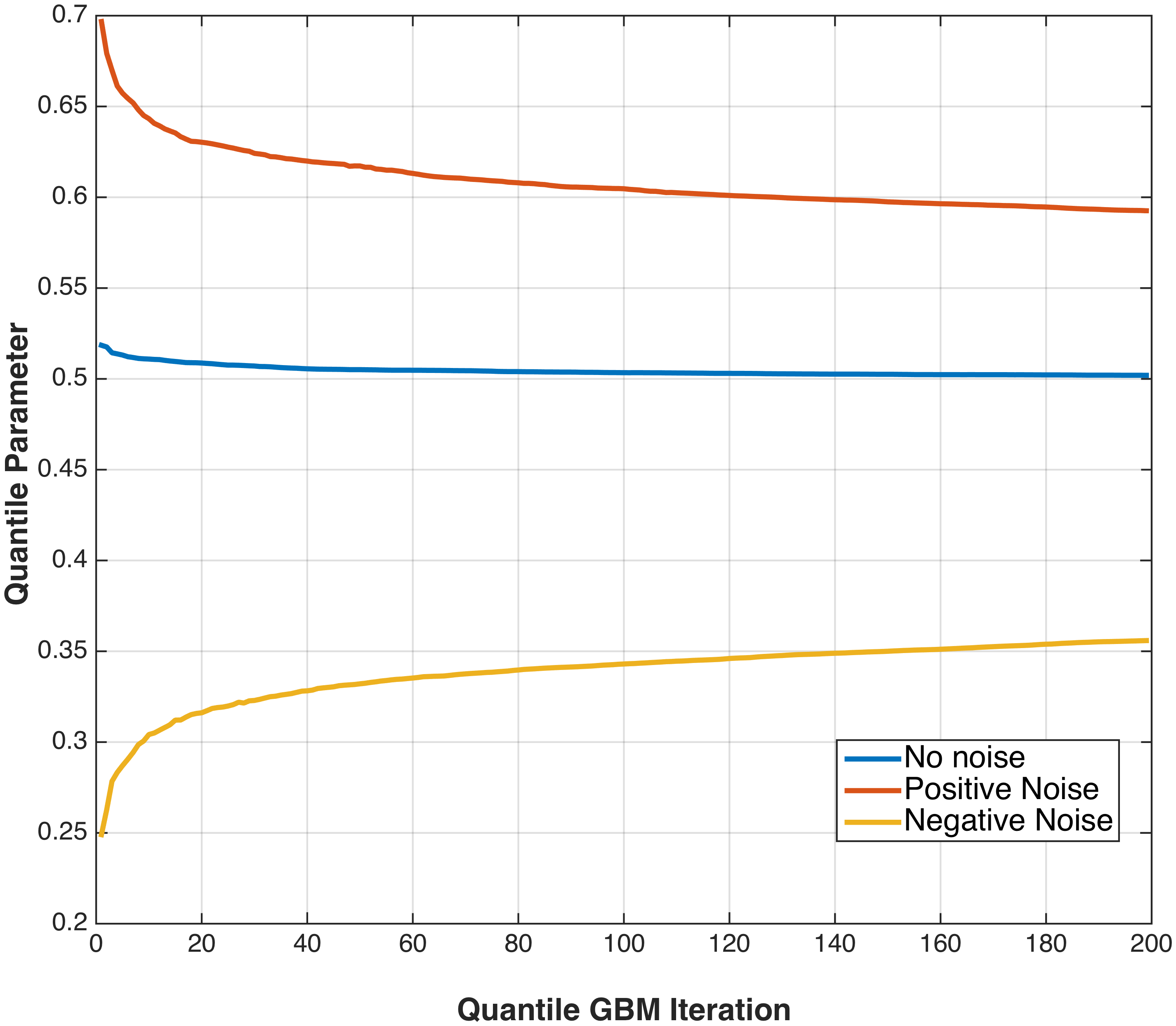}    
\end{minipage}
\vfill
\begin{minipage}[b]{\linewidth}
 \centering
\includegraphics[width=7.5cm]{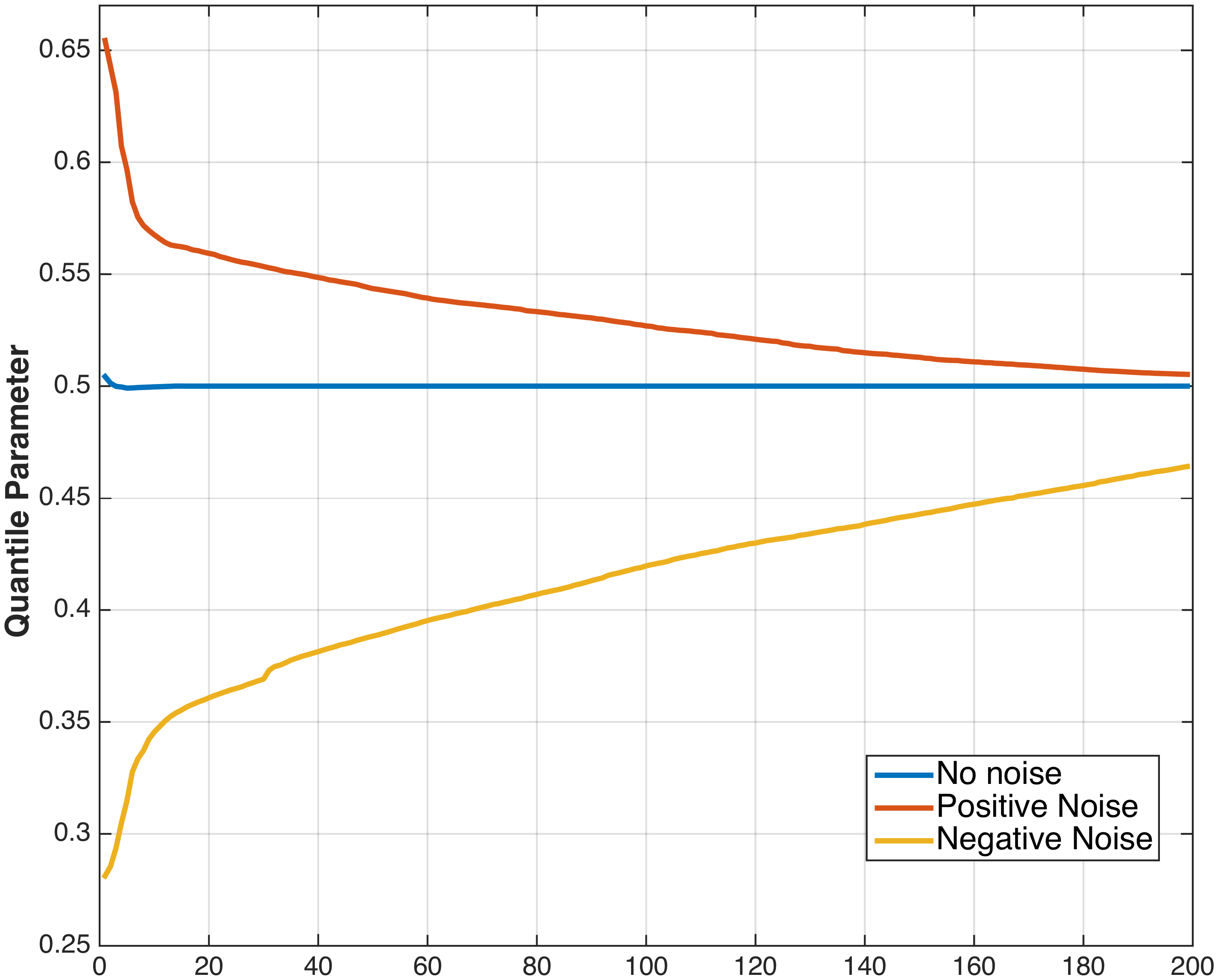}    
\end{minipage}
\vfill
\begin{minipage}[b]{\linewidth}
 \centering
\includegraphics[width=7.5cm]{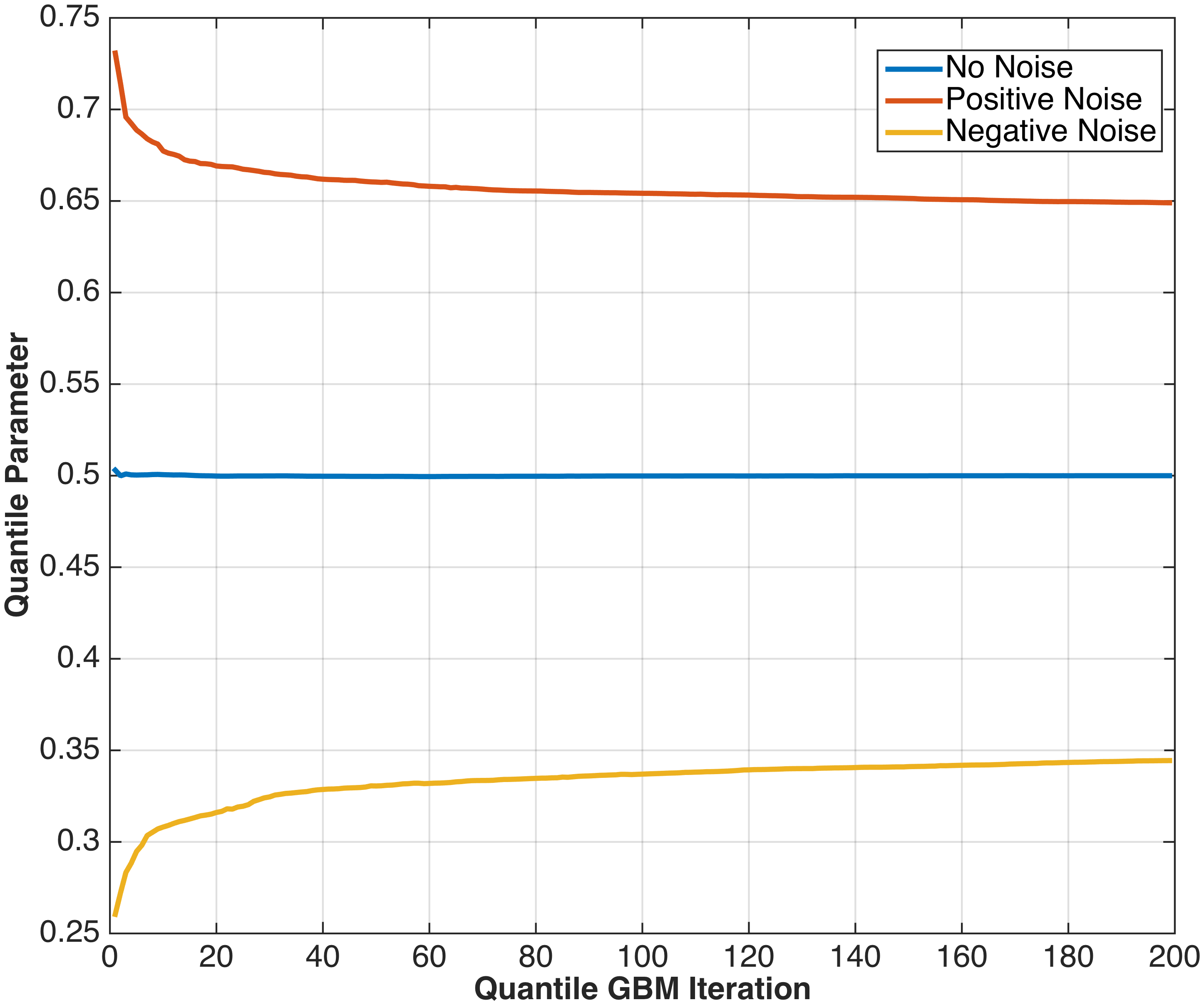}  
\end{minipage} 
\vfill 
\caption{Quantile parameter $\tau$ estimated in each step of the proposed GBM algorithm: (top) Forest fires dataset, (middle) Concrete slump dataset and, (bottom) Compressive strength dataset.}
\label{fig:gbm}
\end{figure}

\begin{table}[t]
  \centering
  \caption{Performance of the proposed quantile GBM algorithm for $4$ UCI datasets at different noise levels and patterns. Regardless of the underlying noise pattern, our approach can identify the appropriate quantiles for recovering the data. }
    \begin{tabular}{|c|c|c|c|}
    \hline
    \textbf{Dataset} & $\boldsymbol{\sigma} $ & \textbf{\%pos./ \%neg.} & \textbf{MSE} \\
    \hline
    \multirow{7}{*}{Forest Fires} & 0     & -   & 0.364 \\
    \cline{2-4}
          & \multirow{3}{*}{2} & 75/25 & 1.17 \\
          &       & 50/50 & 1.12 \\
          &       & 25/75 & 1.15 \\
          \cline{2-4}
          & \multirow{3}{*}{4} & 75/25 & 2.25 \\
          &       & 50/50 & 2.21 \\
          &       & 25/75 & 2.31 \\
     
    \hline
    \multirow{7}{*}{Housing} & 0     & -   & 0.74 \\
    	\cline{2-4}
          & \multirow{3}{*}{2} & 75/25 & 1.39 \\
          &       & 50/50 & 1.44 \\
          &       & 25/75 & 1.32 \\
          \cline{2-4}
          & \multirow{3}{*}{4} & 75/25 & 2.08 \\
          &       & 50/50 & 2.19 \\
          &       & 25/75 & 2.15 \\
          \hline
    \multirow{7}{*}{Concrete Slump} & 0     & -   & 0.19 \\
    \cline{2-4}
          & \multirow{3}{*}{2} & 75/25 & 0.95 \\
          &       & 50/50 & 0.89 \\
          &       & 25/75 & 0.92 \\
          \cline{2-4}
          & \multirow{3}{*}{4} & 75/25 & 1.83 \\
          &       & 50/50 & 1.91 \\
          &       & 25/75 & 1.86 \\
          \hline
    \multirow{7}{*}{Comp. strength} & 0     & -  & 0.48 \\
    \cline{2-4}
          & \multirow{3}{*}{2} & 75/25 & 1.57 \\
          &       & 50/50 & 1.51 \\
          &       & 25/75 & 1.49 \\
          \cline{2-4}
          & \multirow{3}{*}{4} & 75/25 & 2.53 \\
          &       & 50/50 & 2.41 \\
          &       & 25/75 & 2.47 \\
    \hline
    \end{tabular}%
  \label{tab:perf}%
\end{table}%

\subsection{GBM with Quantile Parameter Inference}
GBMs with quantile losses \cite{zheng2012qboost} have been shown to perform better in inferring the quantiles compared to conventional quantile regression. They have a computational advantage, since they proceed by fitting functions greedily to negative gradients. Their hypothesis class is also much larger since each stage of the GBM can potentially use a different base learner. However, since they cannot automatically infer quantile parameters, we have to perform estimation over many values of $\tau$ in order to understand their performance profile and choose the ``best'' quantile parameter according to our requirements.

We propose to extend GBMs with quantile Huber losses to perform automatic inference of the base-learner parameters and the quantile parameters. The greedy stagewise regression (\ref{eqn:greedy_update}) modifies to
\begin{equation}
\label{eqn:greedy_update_quantile}
(\tau_j,\beta_j, \mathbf{z}_j) = \argmin_{\tau, \beta,\mathbf{z}} \sum_{i=1}^n 
L (y_i, f_{j-1} (\mathbf{x}_i)+ \beta \psi_j(\mathbf{x}_i | \mathbf{z})).
\end{equation} In practice this update is first performed by computing the negative gradient in (\ref{eqn:neg_grad}), where the loss $L = \rho_{\tau_{j-1}}$. Here $\tau_{j-1}$ is the quantile parameter estimated in the previous iteration. The parameter $\mathbf{z}_j$ is computed using (\ref{eqn:zupdate}). The base learner coefficient $\beta_j$ and the quantile parameter $\tau_j$ are obtained using the minimization (\ref{eqn:greedy_update_quantile}) with $\mathbf{z}_j$ fixed. This is performed using our proposed algorithm \ref{alg:joint}. Note that the $\mathbf{x}$ in the algorithm corresponds to the one-dimensional parameter $\beta_j$ and $\tau$ corresponds to the quantile parameter $\tau_j$.

\subsection{Experimental Results}
We study the behavior of the proposed quantile GBM algorithm using from $4$ different datasets obtained from the UCI repository: (a) Housing ($506$ samples with $12$ input variables), (b) Compressive strength ($1030$ samples in $9$ input variables), (c) Concrete slump ($103$ samples with $7$ input variables) abd (d) Forest fire ($517$ samples with $10$ input variables). In each case, we used $75\%$ samples for training and the rest for validation, and all reported results were averaged over $10$ independent runs. We believe that the proposed quantile parameter estimation technique can infer the underlying noise model and thereby build robust regression models without user interference. In particular, the we study the regression performance and the behavior of the qunatile estimation process at different levels of positive and negative noise corrupting the data. We rescale all datasets to the range $(-1,1)$ and in each case, we add random Laplacian noise at $3$ different noise levels, $\sigma = \{0, 2, 4\}$ respectively. We fixed the parameter $\kappa = 0.05$, the number of estimators $m = 200$. Since the GBM framework is flexible to used with any base learner, we employed regression trees in our setup.

We evaluated the average Mean Squared Error (MSE) for each of the cases using our proposed quantile GBM algorithm, and they are  reported in Table \ref{tab:perf}. It can be observed from the results that the proposed quantile GBM algorithm can automatically infer the appropriate parameter $\tau$ in each iteration of the algorithm and hence can result in an effective ensemble model. This is verified by corrupting the response variable for a subset of samples with positive noise and the rest with negative noise, and evaluating the regression performance using a validation dataset. As the results show, regardless of the noise pattern considered, for a given $\sigma$, the quantile GBM technique produces very similar MSE. However, it must be noted that the optimal quantile parameter for each of the noise pattern settings is widely different. Figure \ref{fig:gbm} illustrates the quantile parameters chosen in each iteration of the GBM algorithm with different datasets. In particular, we show the results for the cases with completely positive, and completely negative noise components at $\sigma = 4$ and the noise-free case. There are two interesting observations: (a) higher $\tau$ values are chosen when data is corrupted by positive noise and lower $\tau$ values when it is corrupted by negative noise, (b) As the GBM algorithm proceeds towards convergence, it moves closer to the median quantile, indicating that the initial set of estimators in the ensemble have already recovered the underlying data, and the subsequent models added being modeling the noise components.




\bibliography{quantileInference.bib}

\end{document}